\documentclass[twoside,11pt]{article}

\usepackage{jmlr2e}
\usepackage{epstopdf} 
\usepackage{subfigure} 
\usepackage{xcolor}
\usepackage{colortbl}
 
\newcommand\sign{\text{sign}} 
\newcommand\tp{\text{TP}} 
\newcommand\fp{\text{FP}} 
\newcommand\fn{\text{FN}} 
\newcommand\tn{\text{TN}} 
\newcommand\tph{\widehat{\text{TP}}}
\newcommand\fph{\widehat{\text{FP}}}
\newcommand\fnh{\widehat{\text{FN}} }
\newcommand\tnh{\widehat{\text{TN}}}
\newcommand\y{\mathbf{y}} 
\newcommand{\inner}[2]{\langle#1, #2\rangle}
\newcommand{\binner}[2]{\bigg\langle#1, #2\bigg\rangle}
\newcommand\x{\mathbf{x}} 
\newcommand\w{\mathbf{w}} 
\newcommand\h{\mathbf{h}}

\newcommand\bu{\mathbf{u}} 
\newcommand\bv{\mathbf{v}} 
 
\newcommand\X{\textsf{X}} 
\newcommand\Y{\textsf{Y}} 
\newcommand\Yh{\hat{\textsf{Y}}}
\newcommand\Z{\textsf{Z}}
\newcommand\W{\textsf{W}} 
 
\newcommand\Yhij{\hat{\textsf{Y}}_{ij}} 
\newcommand\mP{\mathcal{P}} 
\newcommand\bbP{\mathbb{P}} 
\newcommand\Yij{\textsf{Y}_{ij}}

\newcommand\E{\mathbb{E}}

\newcommand\reg{\text{Reg}} 
\newcommand\rk{\textsf{ rank}} 
 
\newcommand\risk{\text{Risk}}

\usepackage{url}
\usepackage{xcolor}
\usepackage{algorithm,algorithmic}
\usepackage[utf8]{inputenc} 
\usepackage[T1]{fontenc}    
\usepackage{hyperref}       
\usepackage{url}            
\usepackage{booktabs}       
\usepackage{amsfonts}       
\usepackage{nicefrac}       
\usepackage{microtype}      
\usepackage{amsmath,amsfonts,amssymb,amsthm}
\newtheorem{lemma}{Lemma}
\newtheorem{theorem}{Theorem}

\newtheorem{assumption}{Assumption}
\newtheorem{remark}{Remark}
\newtheorem{definition}{Definition}
\newtheorem{corollary}{Corollary}

\ShortHeadings{Regret Bounds for Non-decomposable Metrics with Missing Labels}{Jain and Natarajan}
\firstpageno{1}

\begin{document}

\title{Regret Bounds for Non-decomposable Metrics with Missing Labels}

\author{\name Prateek Jain \email prajain@microsoft.com 
       \AND
       \name Nagarajan Natarajan \email t-nanata@microsoft.com \\
       \addr Microsoft Research, INDIA.}


\maketitle
\begin{abstract}
We consider the problem of recommending relevant labels (items) for a given data point (user). In particular, we are interested in the practically important setting where the evaluation is with respect to non-decomposable (over labels) performance metrics like the $F_1$ measure, \emph{and} the training data has missing labels. To this end, we propose a generic framework that given a performance metric $\Psi$, can devise a regularized objective function and a threshold such that all the values in the predicted score vector above and only above the threshold are selected to be positive.  We show that the regret or generalization error in the given metric $\Psi$ is bounded ultimately by estimation error of certain underlying parameters. In particular, we derive regret bounds under three popular settings: a) collaborative filtering, b) multilabel classification, and c) PU (positive-unlabeled) learning.  For each of the above problems, we can obtain precise non-asymptotic regret bound which is small even when a large fraction of labels is missing. Our empirical results on synthetic and benchmark datasets demonstrate that by explicitly modeling for missing labels and optimizing the desired performance metric, our algorithm indeed achieves significantly better performance (like $F_1$ score) when compared to methods that do not model missing label information carefully. \end{abstract}


\section{Introduction}
Predicting relevant labels/items for a given data point is by now a standard task with applications in several domains like recommendation systems~\citep{koren2009matrix}, document tagging, image tagging~\citep{prabhu2014fastxml}, etc. Many times, like say in collaborative filtering, features for the data points might not be available and one needs to predict labels only on the basis of past labels (e.g., existing likes/dislikes for various labels/items). In presence of features, the problem is the standard multi-label classification problem. 

Design and analysis of algorithms for such tasks should counter two fundamental challenges: a) in practical scenarios, desired performance metric for our predictions are typically complex {\em non-decomposable} functions such as $F_1$ score or precision@$k$; standard metrics like Hamming loss or RMSE over the labels may not be useful, and b) any realistic system in this domain should be able to handle missing labels. Furthermore, often the location of missing labels may not be available like in the positive-unlabeled learning setting~\citep{hsieh2015pu}. Dealing with missing labels may necessitate imposition of certain regularization on the parameters like, say, low-rank regularization so as to exploit the correlations between labels. 

Most of the existing solutions only address one of the two aspects. For example, \citet{koyejo2015consistent} establish that for a large class of performance metrics, the optimal solution is to compute a score vector over all the labels and selecting all the labels whose score is greater than a constant. Their algorithm  treats each label as independent to estimate class-conditional probability separately for each label. Clearly, such methods ignore available information about other labels, and hence cannot handle missing information effectively. Also, such methods do not even apply for the collaborative filtering setting. On the other hand, most of the existing collaborative filtering/matrix completion methods only focus on decomposable losses like RMSE, sum of logistic loss \citep{lafond2015,yu2014large}, which are not effective in real-world systems with large number of labels~\citep{prabhu2014fastxml}. 

In this work, we devise a simple and generic framework that addresses both the aforementioned issues; the framework leads to simple and efficient algorithms in several different settings and for a wide variety of performance metrics used in practice including the multi-label $F$-measure. Our framework is motivated by a simple observation that has been used in other contexts as well \citep{kotlowski2015surrogate,koyejo2015consistent}: for a large class of metrics $\Psi$, simply thresholding the class probability vector leads to bayes-optimal estimators. Hence, the goal would be to estimate per-label class probabilities accurately. To this end, we show that by using a $\lambda$-strongly proper loss along with appropriate thresholding leads to bounded regret wrt. $\Psi$ (Theorem~\ref{thm:mainres}). Note that the threshold can be learned using cross-validation over a small fraction of the training data.

Moreover, $\lambda$-strong convexity of the loss function ensures that by minimizing a nuclear-norm regularized ERM (with risk measured by the selected loss function) wrt. a parameter matrix $\W\in \mathbb{R}^{d \times L}$, we can bound the regret in $\Psi$ by regret in estimation of the optimal $\W$ (Theorem~\ref{thm:mainres}); here, $d$ is the dimensionality of the data and is equal to number of users in case of recommender system. Hence, this result allows us to focus on estimation of $\W^*$ in various different settings such as: a) one-bit matrix completion (Theorem \ref{thm:lafondbased}), popularly used in recommender systems with only like/dislike information, b) one-bit matrix completion with PU learning (Theorem \ref{thm:pucorollary}) applicable to recommender systems where only ``likes" or positive feedback is observed, and c) general multi-label learning with missing labels (Theorem \ref{thm:logisticwithfeatures}). 

For one-bit matrix completion (and the related PU setting), we obtain our final regret bound by adapting existing results from \citet{lafond2015} and \citet{hsieh2015pu}, respectively. For general multilabel setting, a direct application of existing results, such as \citep{lafond2015} leads to {\em weak} bounds. A main technical contribution of our work is to analyze the parameter estimation problem in this setting and provide tight regret bounds. In fact, our result strictly generalizes the result by \citet{lafond2015}, which is for general matrix completion with exponential family noise, to the general {\em inductive matrix completion} setting~\citep{jain2013provable} with exponential family noise. Hence, it should have applications beyond our framework as well.  Finally, we illustrate our framework and algorithms on synthetic as well as real-world datasets. Our method exhibits significant improvement over a natural extension of the method by \citet{koyejo2015consistent} that optimizes $\Psi$ directly but ignores label correlations, hence does not handle missing labels in a principled manner. For example, our method achieves $12\%$ higher $F_1$-measure on a benchmark dataset than that by~\cite{koyejo2015consistent}. 

\paragraph{Related Work.} We now highlight some related theoretical work in recommender systems and multi-label learning.~\cite{gao2013consistency} study consistency and surrogate losses for two specific losses namely Hamming and expected (partial) ranking losses, and leave the other losses to future work. \citet{DembczynskiKH12} consider expected pairwise ranking loss in multilabel learning, show that the problem decomposes into independent binary problems, and provide regret bound for the same. \cite{yun2014ranking} consider the learning to rank problem, where the goal is to rank the relevant labels for a given instance. They show that  popular ranking losses like NDCG can be written as a generalization of certain robust binary loss functions, although they do not provide any explicit regret bounds. Existing theoretical guarantees for 1-bit matrix completion methods used in recommender systems focus solely on RMSE or 0-1 loss~\citep{lafond2015,hsieh2015pu}. 

\section{Problem Setup and Background}
Let $\x_i \in \mathcal{X} \subseteq \mathbb{R}^d$ denote instances and $\y_i \in \{0,1\}^{L}$ denote label vectors. Let $\Y \in \{0,1\}^{n \times L}$ denote the label matrix, with $\y_i$'s as rows.  
In typical multi-label learning and recommender system settings a) the labeling process has some inherent uncertainty, which is usually captured by assuming a conditional distribution $\bbP(\y_{i}|\x_{i})$, b) furthermore, we do not get to observe all the entries of $\y_i$, but only a small subset, say $\Omega_i$. Formally, let $\Omega \subset [n] \times [L]$ denote a subset of indices sampled i.i.d. from a fixed distribution $\pi$ over $[n] \times [L]$. We consider the following sampling model for observing label matrix $\Y$: \\
\begin{equation}
\Y_{ij} = \begin{cases}1 &\mbox{ with probability  } g_j(\x_i; \W^*) \\ 
0 & \mbox{ with probability } 1- g_j(\x_i; \W^*) \end{cases} \text{ for }(i,j) \in \Omega.
\label{eqn:samplingmodel}
\end{equation}
where $\W^*$ parameterizes the underlying conditional distribution $\bbP(\y_i|\x_i)$. Following the low-rank \emph{inductive} matrix completion model~\citep{yu2014large,zhong2015efficient}, we let $\W^{*} \in \mathbb{R}^{d \times L}$ be the parameter matrix and $g_j(\x_i; \W^*) = g(\inner{\x_i}{\w^*_j})$ where $\w^*_j$ is the $j$th column of $\W^*$ corresponding to the $j$th label, for some differentiable function $g: \mathbb{R} \to [0,1]$. A popular choice of $g$ is given by $g(\inner{\x_i}{\w_j}) = \frac{\exp(\inner{\x_i}{\w_j})}{1+\exp(\inner{\x_i}{\w_j})}$, which corresponds to the logistic regression model. When we do not observe feature vectors $\x$, as in the classical recommender system or matrix completion setting, the above model \eqref{eqn:samplingmodel} reduces to the widely studied 1-bit matrix completion model~\citep{cai2013max,davenport20141}:
\begin{equation}
\Y_{ij} = \begin{cases}1 &\mbox{ with probability  } g(\W^*_{ij}) \\ 
0 & \mbox{ with probability } 1- g(\W^*_{ij}) \end{cases} \text{ for }(i,j) \in \Omega,
\label{eqn:samplingmodelmatcomp}
\end{equation}
where $\W^* \in \mathbb{R}^{n \times L}$ is the parameter matrix that captures user-item preferences.

The goal is to learn a multi-label classifier $f: \mathcal{X}^{n} \to \{0,1\}^{n \times L}$ jointly over $n$ instances. The training data consists of input features $\X \in \mathbb{R}^{n \times d}$ where each row corresponds to an instance, drawn iid from some distribution $\mathbb{P}_{\mathcal{X}}$ over $\mathcal{X}$, and \emph{partially observed} label matrix $\Y$ using the sampling model \eqref{eqn:samplingmodel} or \eqref{eqn:samplingmodelmatcomp}, such that a performance metric of interest $\Psi$ is maximized. In this work, we consider a large family of non-decomposable metrics~\citep{koyejo2015consistent} that constitutes linear-fractional functions of (multi-label analogues of) true positives, false positives, false negatives and true negatives defined below. Let $\Yh \in \{0,1\}^{n\times L}$ denote the predicted labels, i.e. $\Yh = f(\X)$ for some $f$. Define the primitives:
\begin{eqnarray*}
\tph_{ij}(\Yh, \Y) = [[\Yhij = 1, \Yij = 1 ]], & \fph_{ij}(\Yh, \Y) = [[\Yhij = 1, \Yij = 0]],\\
\tnh_{ij}(\Yh, \Y) = [[\Yhij = 0, \Yij = 0]], & \fnh_{ij}(\Yh, \Y) = [[\Yhij = 0, \Yij = 1]].
\end{eqnarray*}
For convenience, we drop the arguments and just write $\tph_{ij}$ to denote $\tph_{ij}(\Yh, \Y)$ and so on. 

1. \textbf{Micro-averaged metrics}. Define: \[ \tph(\Yh, \Y) =\frac{1}{|\Omega|}\sum_{(i,j)\in\Omega}\tph_{ij}, \] and $\fph(\Yh, \Y), \tnh(\Yh, \Y), \fnh(\Yh, \Y)$ similarly. Let $\tp = \E[\tph], \fp = \E[\fph]$ (and so on), where the expectation is defined wrt to the sampling distribution $\pi$ over indices $[n] \times [L]$ as well as the joint distribution over instances and labels. 
Micro-averaged performance metric $\Psi: \{0,1\}^{n, L} \times \{0,1\}^{n, L} \to \mathbb{R}_+$ is given by:
\begin{equation}
\label{eqn:family1}
\Psi(\Yh, \Y) = \frac{a_0 + a_{11} \tp + a_{01} \fp + a_{10} \fn + a_{00} \tn}{b_0 + b_{11} \tp + b_{01} \fp + b_{10} \fn + b_{00} \tn} .
\end{equation}
for bounded constants $a$'s and $b$'s. Assume that $\Psi$ is bounded, i.e. $\exists \gamma > 0$ such that $b_0 + b_{11} \tp + b_{01} \fp + b_{10} \fn + b_{00} \tn > \gamma$ for all $\Yh, \Y$.\\
2. \textbf{Instance-averaged metrics}. 
Define \[\tph_i(\Yh, \Y) = \frac{1}{|\Omega_i|}\sum_{j \in \Omega_i} \tph_{ij}.\] Let $\tp_i = \E[\tph_i]$.  Instance-averaged performance metric $\Psi$ is given by:
\begin{equation}
\label{eqn:family2}
\Psi(\Yh, \Y) = \frac{1}{n}\sum_{i=1}^n \frac{a_0 + a_{11} \tp_i + a_{01} \fp_i + a_{10} \fn_i + a_{00} \tn_i}{b_0 + b_{11} \tp_i + b_{01} \fp_i + b_{10} \fn_i + b_{00} \tn_i} .
\end{equation}
for bounded constants $a$'s and $b$'s. Assume that $\Psi$ is bounded, i.e. $\exists \gamma > 0$ such that $b_0 + b_{11} \tp_i + b_{01} \fp_i + b_{10} \fn_i + b_{00} \tn_i > \gamma$ for all $\Yh, \Y, i$.\\
3. \textbf{Macro-averaged metrics}. 
Let $\Omega^{(j)} = \{ i: (i,j) \in \Omega\}$. Define: \[\tph_j(\Yh, \Y) = \frac{1}{|\Omega^{(j)}|} \sum_{i \in \Omega^{(j)}}\tph_{ij}.\] Let $\tp_j = \E[\tph_j]$. Macro-averaged performance metric $\Psi$ is given by:
\begin{equation}
\label{eqn:family3}
\Psi(\Yh, \Y) = \frac{1}{L}\sum_{j=1}^L \frac{a_0 + a_{11} \tp_j + a_{01} \fp_j + a_{10} \fn_j + a_{00} \tn_j}{b_0 + b_{11} \tp_j + b_{01} \fp_j + b_{10} \fn_j + b_{00} \tn_j} .
\end{equation}
for bounded constants $a$'s and $b$'s. Assume that $\Psi$ is bounded, i.e. $\exists \gamma > 0$ such that $b_0 + b_{11} \tp_j + b_{01} \fp_j + b_{10} \fn_j + b_{00} \tn_j > \gamma$ for all $\Yh, \Y, j$.

\textbf{Example metrics}:
\begin{enumerate}
\item Instance-averaged $F_1$ metric defined as:
$\Psi_{F_1}(\Yh, \Y) = \frac{1}{n} \sum_{i=1}^n \frac{2 \tp_i}{2\tp_i + \fp_I + \fn_i}$.
\item Accuracy (equivalent to the Hamming loss):
$\Psi_{\text{Ham}}(\Yh, \Y) = 1 - \frac{1}{n} \sum_{i=1}^n {\fp_i + \fn_i}$.
\end{enumerate}
\begin{remark}
The aforementioned definitions of performance metrics naturally apply to the recommender system setting, where data is observed via the 1-bit matrix completion sampling model \eqref{eqn:samplingmodelmatcomp}. Here, the recovery error is ultimately measured wrt to an estimated binary-valued matrix. Note that in this case, the expectations are defined wrt the sampling distribution $\pi$ and the inherent noise in 1-bit sampling $\bbP(\Y_{ij}|\W_{ij})$. 
\end{remark}

Let $\Psi^*$ denote the Bayes optimal performance, i.e. $\Psi^* = \max_f \Psi(f(\X), \Y)$ (Note that $\Psi$ is defined in terms of expectation with respect to the underlying distribution). Our objective can be now stated learning $\hat{f}$ such that the $\Psi$-regret, i.e. $\Psi^* - \Psi(\hat{f}(\X), \Y)$, is provably bounded.~\cite{koyejo2015consistent} showed that the Bayes optimal $\Psi^{*}$ thresholds the conditional probability of each label $j$, i.e. $\bbP(y_{j}|\x)$ at a certain value $\delta^{*} \in (0,1)$, and that the value $\delta^{*}$ is shared across all the labels.\footnote{The definitions in ~\citep{koyejo2015consistent} do not include general sampling distribution $\pi$, but the results can be generalized in a straight-forward manner.}:

\section{Algorithm}
Our approach is based on estimating real-valued predictions and then thresholding the predictions optimally in order to maximize a given metric $\Psi$. ~\cite{koyejo2015consistent} proposed a simple consistent plug-in estimator algorithm, which first computes conditional marginals $\bbP(y_{j}| \x)$ independently for each label $j$, and then estimates a threshold jointly to optimize $\Psi$. While the approach is provably consistent asymptotically, it is not clear if it admits a useful regret bound; in particular, we would like to characterize the behavior in the finite samples regime. In case of the sampling model \eqref{eqn:samplingmodel}, the approach translates to learning columns of the parameter matrix $\W$ independently. In many cases, $\W$ exhibits some structure, such as low-rankness, reflecting correlation between labels~\citep{yu2014large,zhong2015efficient,davenport20141}. Statistically, capturing correlations via a low-rank structure could help improve the sample complexity for recovery, and computationally, it would help reduce space and time complexity of the learning procedure. 

Our proposed algorithm is presented in Algorithm \ref{algo:algo1}. In Step 1, we solve a trace-regularized minimization problem to estimate the parameter matrix $\W$, where the function $\ell$ can be any bounded loss such as the squared, the logistic or the squared Hinge loss. In particular, using the logistic loss corresponds to the maximum likelihood estimation of the sampling model \eqref{eqn:samplingmodel}.~\citet{yu2014large} also solve essentially the same objective as \eqref{eqn:objective},  except for the additional bound constraint on entries of $\X\W$. The optimization problem \eqref{eqn:objective} can be solved using a proximal gradient descent algorithm, with a fast proximal operator computation by storing the current solution in a low-rank form. We could also use fast non-convex procedure, by writing $\W = \W_1\W_2^T$, where $\W_1$ and $\W_2$ are low-rank matrices with $k \ll \min(d, L)$ columns each, and applying alternating minimization.

The real-valued estimator is given by $\Z = \X\hat{\W}$ in Step 2. To obtain binary-valued predictions, we solve a 1-dimensional optimization problem to compute the optimal threshold, on the training data. Note that this step can be done in $|\Omega|$ time. 
\begin{remark} In the 1-bit matrix completion setting, we obtain a thresholded max-likelihood estimator of $\W^* \in \mathbb{R}^{n \times L}$ using identical procedure; where we interpret $\X$ in Algorithm \ref{algo:algo1} as the identity matrix of size $n$.\end{remark}

\begin{algorithm}
\caption{Thresholded Max-Likelihood Estimator}
\label{algo:algo1}
\begin{algorithmic}
\STATE \textbf{Input}: Training data $\X \in \mathbb{R}^{n \times d}$, labels $\Y_{\Omega} \in \{0,1\}^{n \times L}$ observed on indices $\Omega$ and metric $\Psi$.
\STATE 1. Obtain $\hat{\W}$ by solving the trace-constrained matrix completion:
\begin{equation}
\hat{\W} = \arg \min_{\W: \|\X\W\|_\infty \leq \gamma} \frac{1}{|\Omega|} \sum_{(i,j) \in \Omega} \ell (\inner{\x_{i}}{\w_{j}}, \Y_{ij}) + \lambda \|\W\|_{*},
\label{eqn:objective}
\end{equation}
\STATE 2. Let $\Z = \X\hat{\W}$. Define the thresholding operator $\Yh = \text{Thr}_\theta(\Z)$, such that $\Yhij = [\![ \Z_{ij} \geq \theta ]\!]$. 
\STATE 3. Return $\Yh =  \text{Thr}_{\hat{\theta}}(\Z)$, where 
\[ \hat{\theta} = \arg \max_{\theta} \Psi(\text{Thr}_\theta (\Z_{\Omega}), \Y_{\Omega}). \]
\end{algorithmic}
\end{algorithm}

\section{Analysis: Regret Bounds}
In this Section, we first show that $\Psi$-regret can be bounded with the regret of a certain loss $\ell$. Then, under various sampling models pertaining to different settings such as 1-bit matrix completion,  multi-label learning, and PU (positive-unlabeled) learning, we show that the $\ell$-regret can be bounded, via recovering the underlying parameter matrix governing $\bbP(y_{ij}|\x_i)$.
\subsection{Low $\ell$-regret implies low $\Psi$-regret}
Our first main result connects $\Psi$-regret to regret with respect to a strongly proper loss function $\ell$~\citep{agarwal2014surrogate}. Canonical examples of strongly proper losses include the logistic loss $\ell(t, y) = \log(1 + \exp(-yt))$,  the exponential loss $\ell(t, y) = \exp(-yt)$ and the squared loss $\ell(t,y) = (1 - yt)^2$. Define the $\ell$-regret of $\Z \in \mathbb{R}^{n \times L}$ as:
\[ \reg_{\ell}(\Z) = \E[\ell(\Z_{ij}, \Yij)] - \min_{\Z' \in \mathbb{R}^{n\times L}} \E[\ell(\Z'_{ij}, \Yij)], \]
where the expectation is wrt. draws from $\pi$ and the joint distribution over instances and labels.

\begin{theorem} [Main Result 1]
\label{thm:mainres}
Let $\Psi$ be a performance metric as defined in \eqref{eqn:family1} , \eqref{eqn:family2} or \eqref{eqn:family3}. Let $\ell$ be a $\lambda$-strongly proper loss function. Assume the input $\X \in \mathbb{R}^{n \times L}$ consists of iid instances sampled from marginal $\bbP_{\mathcal{\X}}$, label matrix $\Y \in \{0,1\}^{n\times L}$, where $y_{ij}$ is sampled iid from $\bbP(y_{ij}|\x_i)$, which is observed only on a subset of indices $\Omega$ sampled iid from a fixed distribution $\pi$. Then, the output $\Yh$ obtained by thresholding the estimate $\Z$ in Step 3 of Algorithm \ref{algo:algo1} satisfies the regret bound:
\begin{equation}
\Psi^* - \Psi(\Yh, \Y) \leq C \sqrt{\frac{2}{\lambda}}\sqrt{\emph{\reg}_\ell(\Z)} + O\bigg(\frac{1}{\sqrt{|\Omega|}}\bigg), 
\label{eqn:mainres}
\end{equation}
for some absolute constants $C$ and $\lambda$.
\end{theorem}
We emphasize that the above result holds for arbitrary metric $\Psi$ from the family \eqref{eqn:family1}, \eqref{eqn:family2} or \eqref{eqn:family3}. Consider the RHS of \eqref{eqn:mainres}: $1/\sqrt{|\Omega|}$ is the lower-order term, and independent of dimensionality; the first term makes the framework fairly powerful, as it can use any strongly proper loss. In the next subsection, we will provide precise instantiations of this term under various learning settings.
\paragraph{Proof Outline for Theorem \ref{thm:mainres}.}  Proof technique is based on \citep{kotlowski2015surrogate}, where they derive similar bound in the binary classification setting. We first relate the $\Psi$-regret to weighted 0-1 loss regret (Lemma \ref{lem:weightedriskequiv}). Then, we show there exists a thresholding $\text{Thr}_{\theta^*}(\Z) \in \{0,1\}^{n\times L}$ such that its weighted loss regret is bounded by the $\ell$-regret of a strongly proper loss $\ell$ (Lemma \ref{lem:thresregretbound}). Finally, we argue that it suffices to estimate $\hat{\theta}$ from the training data (Lemma \ref{lem:estimatetheta}). Detailed proof and associated Lemmas are available in Appendix \ref{app:thm1}. \qed \\
\subsection{Bounding $\ell$-regret}
Below, we provide the desired $\ell$-regret bound under three different settings.
\subsubsection{Collaborative Filtering}
Consider the 1-bit matrix completion sampling model in~\eqref{eqn:samplingmodelmatcomp}. Then \eqref{eqn:objective} reduces to the optimization problem considered by \citet{lafond2015}. We have the following regret bound for the estimator $\Z = \hat{\W}$ obtained in Step 2 of Algorithm \ref{algo:algo1} (Note that $\X$ is just treated as identity in this setting). 
\begin{theorem}
Assume $\pi$ is uniform, and consider the 1-bit matrix completion sampling model \eqref{eqn:samplingmodelmatcomp}. Let $\ell$ denote a $1$-Lipschitz, strongly proper loss (appearing in \eqref{eqn:mainres}), and $\Z$ denote the output of Step 2 of Algorithm \ref{algo:algo1}. With probability at least $1 - \delta$, the following holds:
\[ \emph{\reg}_{\ell}(\Z) \leq \sqrt{\tilde{C}\max\bigg(\frac{\max(n,L) \rk(\W^*)\log (3/\delta)}{|\Omega|} \bigg(\sigma_\gamma^2 +1 \bigg), \gamma^2\sqrt{\frac{\log(3/\delta)}{|\Omega|}}\bigg)}, \]
where $\tilde{C}, c_\gamma, c'_\gamma, c''_\gamma, \sigma_\gamma$ are numerical constants, and $\gamma = \max_{ij} |\W_{ij}^*|$.
\label{thm:lafondbased}
\end{theorem}
Note that when $|\Omega| > \max(n, L)$, the RHS of the above bound starts converging; in particular, the second term within $\max$ is the lower-order term: $\gamma \approx O\big(\sqrt{1/nL}\big)$. Theorem \ref{thm:lafondbased} can be extended to general distributions $\pi$ beyond uniform, satisfying mild assumptions. See Appendix \ref{app:lafondbased}.

\subsubsection{Multi-label Learning}
Consider the sampling model \eqref{eqn:samplingmodel} with features. We have the following regret bound for the estimator $\Z = \X\hat{\W}$ obtained in Step 2 of Algorithm \ref{algo:algo1}, under the following assumptions.
\begin{assumption}
The marginal distribution over the features $\bbP_{\mathcal{X}}$ is sub-Gaussian with sub-Gaussian norm $K$ and covariance $\Sigma \in \mathbb{R}^{d \times d}$.
\label{assume:features}
\end{assumption}
\begin{assumption}
\label{assume:omegadist}
Let $\pi_{k,l}$ denote the probability of sampling the entry $(k,l) \in [n] \times [L]$; \begin{enumerate} 
\item $\exists\ \mu \geq 1$ s.t. 
$\min_{k \in [n], l \in [L]} \pi_{k,l} \geq \frac{1}{\mu n L}$, and
\item $\exists\ \nu \geq 1$ s.t.  $\max_{i', j'} \big( \sum_{j}\pi_{i'j}, \sum_{i}\pi_{ij'} \big) \leq \frac{\nu}{\min(n,L)}$.
\end{enumerate}
\end{assumption}

\begin{theorem} [Main Result 2]
Assume \ref{assume:features}, \ref{assume:omegadist} and consider the sampling model \eqref{eqn:samplingmodel}. Also assume $L \geq d$. Let $\hat{\W}$ be the solution to the trace-norm regularized optimization problem~\eqref{eqn:objective} using logistic loss for $\ell$, number of training data points $n \geq C' \ . \ d$, number of observations  $|\Omega| \geq L + d$, and setting the regularization parameter
$\lambda =  \frac{2c}{\sqrt{|\Omega|}}$.
Then, with probability at least $1 - 3(n + L)^{-1} - 2(d+L)^{-1}$, the following holds:
\[ \frac{\|\hat{\W} - \W^*\|^2_F}{dL} \leq \frac{C_2\mu^2}{d} \max\bigg(\frac{L \rk(\W^*)\log (n+L)}{|\Omega|} \bigg(\sigma_\gamma^2 +1 \bigg), \frac{\gamma^2}{\mu}\sqrt{\frac{\log(n + L)}{|\Omega|}}\bigg), \]
where $c,C', C_2$ are numerical constants and $\sigma_\gamma \leq (1+e^{\gamma})^2e^{\gamma}$.
\label{thm:logisticwithfeatures}
\end{theorem}
A few remarks of our result in the multi-label setting are in order:
 \begin{remark}[Generalization] The result in Theorem~\ref{thm:logisticwithfeatures}, and Theorem~\ref{thm:corollary} in Appendix B for general exponential distributions, is a key technical contribution of this work. In particular, our analysis applies to $\Y$ arising from general exponential distributions, including Gaussian when $\Y$ is real-valued and Poisson when $\Y$ models counts. See Appendix \ref{app:generalexponential} for more details.
\end{remark}
\begin{remark} [Comparing~\citep{lafond2015}] If we directly apply the method and the analysis of~\citep{lafond2015}, the resulting bounds are very weak; in fact, when $n \geq L$ and $|\Omega| = O(n)$, which is quite common in the multi-label scenario, the ensuing bound suggests that the estimator is not even consistent, even when $\pi$ is uniform. See Appendix \ref{sec:weaknesslafond} for details. 
\end{remark}

\begin{remark} [Comparing~\citep{koyejo2015consistent}] The plugin-in estimator algorithm of~\citep{koyejo2015consistent} estimates $\w^*_j$ for each label $j$ independently, and learns a common threshold as in Algorithm \ref{algo:algo1}. Let $\hat{\w}_j$ denote the estimator for label $j$. Then, using standard analysis we have, $\|\hat{\w}_j - \w^*_j\|_2 \leq \sigma \sqrt{\frac{d}{|\Omega^j|}}$, where $|\Omega^j|$ is the number of observations per label which is $O(\frac{|\Omega|}{L})$. Thus we have the bound: $\frac{\|\W^* - \hat{\W}\|_F^2}{L} \leq \sigma O(\frac{Ld}{|\Omega|})$. This is how our bounds behave, when $\W^*$ is indeed full rank, up to constants. When $\rk(\W^*) \ll \min(d, L)$, we achieve much faster convergence.
\end{remark}
We now give the desired $\ell$-regret bound as a corollary.
\begin{corollary}
Assume the conditions of Theorem \ref{thm:logisticwithfeatures} hold. Let $\ell$ denote a $1$-Lipschitz, strongly proper loss (appearing in \eqref{eqn:mainres}), and $\Z = \X\hat{\W}$ denote the output of Step 2 of Algorithm \ref{algo:algo1}. With probability at least $1 - \delta - (d+L)^{-1}$, the following holds:
\[  \emph{\reg}_{\ell}(\Z) \leq \sqrt{C_2\mu^2 \max\bigg(\frac{L \rk(\W^*)\log (3/\delta)}{|\Omega|} \bigg(\sigma_\gamma^2 +1 \bigg), \frac{\gamma^2}{\mu}\sqrt{\frac{\log(3/\delta)}{|\Omega|}}\bigg)}, \]
where $c,C', C_2, \sigma_\gamma$ are defined as in Theorem \ref{thm:logisticwithfeatures}.
\label{thm:multilabelcorollary}
\end{corollary}
\paragraph{Proof Outline for Theorem \ref{thm:logisticwithfeatures}.} We analyze the following general exponential noise model for $\Y$:
\begin{equation}
y_{ij} | \x_i, \w_j \sim \exp_{h, G}(\x_i, \w_j) := h(y_{ij}) \exp\big(\inner{\x_i}{\w_j}y_{ij} - G(\inner{\x_i}{\w_j})\big) , 
\label{eqn:exponentialdist1}
\end{equation}
where $h$ and $G$ are the base measure and log-partition functions associated with this canonical representation. Our proof sketch is based on~\citet{lafond2015}, but requires bounding certain quantities carefully. In particular, we prove a tight bound for $\|\X^T \nabla \Phi_\Y(\X,\W^*)\|_2$ in terms of the regularization parameter $\lambda$, where $\Phi_\Y(\X,\W^*)$ is the MLE wrt. general exponential distribution (reduces to \eqref{eqn:objective}, without regularization, when $y_{ij}$'s are from \eqref{eqn:samplingmodel}), as stated below.
\begin{lemma}
Consider the sampling model \eqref{eqn:exponentialdist1}. Assume (i) $d \leq L$, (ii) $|\Omega| \geq (L+d)$, (iii) $y_{ij}$'s are sampled independently given $\x_i$, and (iv) $|y_{ij} - G'(\inner{\x_i}{\w^*_j})| \leq \alpha$, for all $i, j \in [n] \times [L]$, for any $n, L$. Let $\X \in \mathbb{R}^{n \times d}$ whose rows ($\x_i$'s) are iid samples from $\bbP_{\mathcal{X}}$ satisfying Assumption \ref{assume:features}. Then, with probability at least $1 - (d+L)^{-1}$, there exists numerical constant $c$ such that,
\[ \big\| \X^T \nabla \Phi_Y(\X, \W^*) \big\|_2 \leq c \ . \ \frac{\alpha}{\sqrt{|\Omega|}}. \]
\label{lem:gradbound}
\end{lemma}
\subsubsection{PU Learning}
\label{sec:pulearning}
In many collaborative filtering and multi-label learning tasks, only the positive entries ($y_{ij} = 1$) are observed. In this setting, we can use the approach of~\citep{hsieh2015pu}, where they consider a two-stage sampling model: sample $y_{ij}$ using \eqref{eqn:samplingmodelmatcomp} for all $i, j \in [n] \times [L]$ (or using \eqref{eqn:samplingmodel} when features are available), and then flip a fraction $\rho$ of the sampled 1's to 0's, resulting in $\tilde{\Y}$. We would then use the unbiased estimator $\tilde{\ell}$ of loss $\ell$ in \eqref{eqn:objective}; $\tilde{\ell}$ satisfies $\E[\tilde{\ell}(\Z_{ij},\tilde{\Y}_{ij})] =  \ell(\Z_{ij},\Y_{ij})$, where the expectation is wrt the flipping process, parameterized by $\rho$. For the estimator $\Z = \hat{\W}$ obtained thus, we have the following regret bound.
\begin{theorem}
Let $\ell$ denote a $1$-Lipschitz, strongly proper loss (appearing in \eqref{eqn:mainres}). Assume $\|\W^*\|_* \leq t$. Let $\Z = \X\hat{\W}$, where $\hat{\W}$ is obtained by solving the unbiased estimator objective of~\citet{hsieh2015pu}. With probability at least $1 - \delta$, there exists absolute constant $C$ such that:
\[ \emph{\reg}_\ell(\Z) \leq \sqrt{6\frac{\sqrt{\log(2/\delta)}}{\sqrt{nL}(1-\rho)} + 2 C \ . \ t \frac{\sqrt{n} + \sqrt{L}}{(1-\rho)nL}}. \]
\label{thm:pucorollary}
\end{theorem}
The RHS of the bound above, when $n = L$, is of $O(\sqrt{\frac{1}{n(1-\rho)}})$, where $(1-\rho)$ is the fraction of observed 1's in $\tilde{Y}$. Naturally, as $\rho$ is large, we need more samples to achieve similar rates as in the other settings. 
\begin{remark}
This PU learning result is particularly very useful in extreme classification setting~\citep{bhatia2015sparse,prabhu2014fastxml}; where there are too many labels and is unrealistic to get feedback on every label, but possible to obtain a small subset of relevant labels for instances. Furthermore, the above result serves to attest to the utility of our framework.
\end{remark}

\section{Experiments}
We focus on multi-label datasets for experimental study. The goal is to show that the convergence happens as suggested by the theory, and that the proposed algorithm performs well on real-world datasets. To solve \eqref{eqn:objective}, we use an alternating minimization procedure by forming $\W = \W_1 \W_2^T$, such that $\W_1 \in \mathbb{R}^{d \times k}$ and $\W_2 \in \mathbb{R}^{L \times k}$, where $k$, the rank of $\W$, is an input parameter.

\subsection{Synthetic data}
We generate multi-label data as follows. We fix $n = 1000, L = 100$ and $d = 10$. First, we generate $\X \in \mathbb{R}^{n \times d}$ using samples from multi-variate Gaussian $\mathcal{N}(0, I)$. Then, we generate $\W^*$ of rank $5$. The label matrix $\Y$ is obtained by thresholding $\X\W^*$ at $\theta^* = 0$, i.e. $y_{ij} = \sign(\inner{\x_i}{\w_j^*})$. In this noise-free setting, we expect that our algorithm would recover both $\W^*$ and $\theta^*$ accurately as it sees more and more observations. The results for maximizing micro $F_1$ and accuracy metrics are presented in Figure \ref{fig:synth}. As the sampling ratio $\frac{|\Omega|}{nL}$ increases, we observe that the proposed estimator achieves optimal performance in both the cases. Furthermore, even when only $10\%$ of the observations are revealed, we observe that the proposed method achieves very high $F_1$ as well as accuracy values, compared to learning the columns of $\W^*$ independently via the plugin estimator method proposed by \citep{koyejo2015consistent} (followed by learning a threshold).
\vspace{-0.5cm}
\begin{figure}[h!]
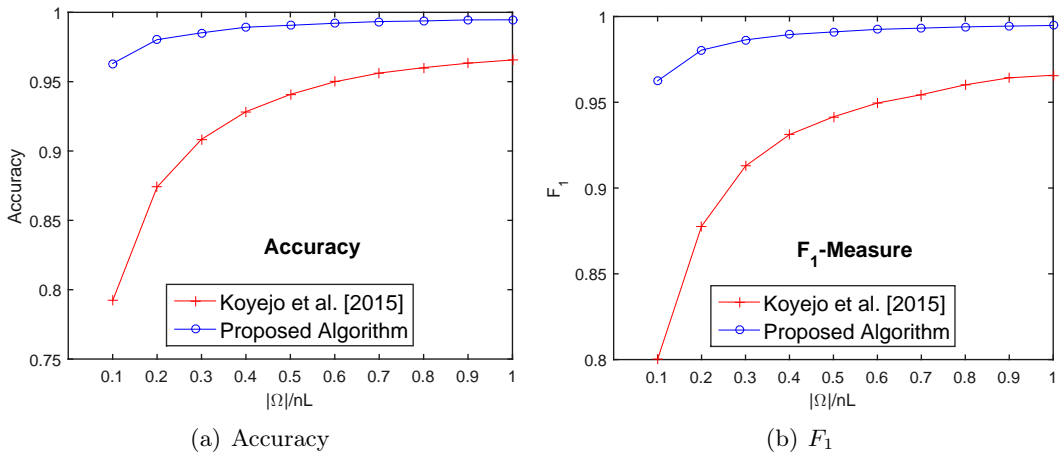

        \centering
        \subfigure[Accuracy]{\includegraphics[width=0.45\textwidth]{hamming.eps}}\hspace{0.2cm}
        \subfigure[$F_{1}$]{\includegraphics[width=0.45\textwidth]{f1.eps}}
\caption{Convergence of the methods for Accuracy and micro-$F_1$ metrics on synthetic data.}
\label{fig:synth}
\end{figure}
\begin{table}[th]
\begin{tabular}{|p{1.5cm}|| r | r  ||  r | r | }\hline
    {\sc Dataset} & \citet{koyejo2015consistent} & Algorithm \ref{algo:algo1}  &\citet{koyejo2015consistent}&  Algorithm \ref{algo:algo1}   \\  
      &micro $F_1$ & micro $F_1$ & Accuracy & Accuracy  \\ \hline \hline
\textsc{CAL500} & \cellcolor{gray!25}{0.4267 $\pm$ 0.0016} & 0.3855 $\pm$ 0.0005 & \cellcolor{gray!25}{0.8541 $\pm$ 0.0034} & 0.8493 $\pm$ 0.0002 \\ \hline
\textsc{Autofood} & 0.4897 $\pm$ 0.0103 &  \cellcolor{gray!25}{0.5597 $\pm$ 0.0047} & 0.9307 $\pm$ 0.0064 & \cellcolor{gray!25}{0.9345 $\pm$ 0.0043} \\ \hline
\textsc{Bibtex} & \cellcolor{gray!25}{0.2641 $\pm$ 0.0251} & 0.2398 $\pm$ 0.0133 & 0.9849 $\pm$ 0.0016 & \cellcolor{gray!25}{0.9856 $\pm$ 0.0003}\\ \hline
\textsc{Compphys} & 0.2463 $\pm$ 0.0315 &  \cellcolor{gray!25}{0.3510 $\pm$ 0.0293} & 0.9448 $\pm$ 0.0011 &  \cellcolor{gray!25}{0.9466 $\pm$ 0.0012} \\ \hline
\textsc{Corel5k} & 0.1552 $\pm$ 0.0116  &  \cellcolor{gray!25}{0.1642 $\pm$ 0.0001} &  \cellcolor{gray!25}{0.9906 $\pm$ 0.0000}  &  \cellcolor{gray!25}{0.9906 $\pm$ 0.0000} \\ \hline
\end{tabular}
\caption{Comparison of proposed algorithm and plugin-estimator method of \citep{koyejo2015consistent} on multi-label micro $F_{1}$ and Hamming (i.e. Accuracy) metrics. Reported values correspond to \emph{micro-averaged} metric computed on test data. In all the cases, $\frac{|\Omega|}{nL}$ was fixed to 20\% for training. The rank of $\W$ was set to $0.4L$ for Algorithm \ref{algo:algo1}. We observe that the proposed algorithm which captures label correlations performs better consistently across datasets. }
\label{fig:results}
\end{table}
\subsection{Real-world data}
We consider five real-world multi-label datasets widely used as benchmarks~\citep{bhatia2015sparse,yu2014large}. 
\begin{enumerate}
\item \textsc{CAL500}: a music dataset with 400 training and 100 test instances, $L = 174$, $d = 68$, 
\item \textsc{Corel5k}: an image dataset with 4500 training and 500 test instances, $L = 374$, $d = 499$, 
\item \textsc{Bibtex}: a text dataset with 4,880 training and 2,515 test instances, $L = 159$, $d = 1,836$, 
\item \textsc{Compphys} dataset with 161 training and 40 test instances, $L = 208$, $d = 33,284$, and 
\item \textsc{Autofood} dataset with 4,880 training and 2,515 test instances, $L = 162$, $d = 1,836$. 
\end{enumerate}

We set the rank $k$ of $\W$ to $0.4L$ for all the datasets in our method, and set $\frac{|\Omega|}{nL} = 20\%$ to train the models in each method. The results are presented in Table \ref{fig:results}. We observe that the proposed method is competitive in all the datasets, and achieves better micro-$F_{1}$ and accuracy values, with a small value of rank $0.4L$. We note that the label matrices of most of the datasets are very sparse (for instance, less than 8.5\% of the test data are positive labels in \textsc{Autofood}), which explains high accuracy and low $F_1$ values. The learned model is much more compact than that of \citep{koyejo2015consistent} ($k(d+L)$ vs $dL$ parameters). While our bounds in theory hold for the case $L \geq d$ (Theorem \ref{thm:logisticwithfeatures}), many of the datasets considered here have $d \geq L$ and yet the performance is competitive.

\section{Conclusions}
We presented a framework for optimizing general performance metrics applicable to multi-label as well as collaborative filtering settings. Our work complements recent results in this direction: on the theoretical front, we derive strong regret bounds for practically used metrics like $F$-measure, and on the algorithmic front, we provide simple and efficient procedure that works well in practice.

\newpage
\bibliography{auc_matrixcompletion}
\newpage
\appendix
\section{Proofs}
\subsection{Proof of Theorem \ref{thm:mainres}}
\label{app:thm1}
Proof technique is based on \citep{kotlowski2015surrogate}, where they derive similar bound in the binary classification setting. We first relate the $\Psi$-regret to weighted 0-1 loss regret. Define the $\alpha$-weighted 0-1 loss $\ell_{\alpha}: \mathbb{R} \times \mathbb{R} \to [0,1]$ as:
\[ \ell_{\alpha}(\hat{y}, y) = \alpha[\![ y = 0 ]\!][\![ \hat{y} = 1 ]\!] + (1-\alpha) [\![ y = 1 ]\!] [\![\hat{y} = 0]\!], \]
Let $\Yh = f(\X)$ for some function $f$. The $\ell_{\alpha}$-risk of $f$ with respect to the underlying distribution over $\X, \Y$ and $\Omega$ is defined as:
\[ \risk_\alpha(\Yh) = \E[\ell_\alpha(\Yhij, \Yij)] = \alpha \fp(\Yh, \Y) + (1-\alpha)\fn(\Yh, \Y). \]
Define the Bayes optimal corresponding to the above risk: $f_\alpha^* = \arg\min_f \risk_\alpha(f(X), \Y)$. Let $\risk_{\alpha}^* := f_\alpha^*(\X)$. The $\ell_{\alpha}$-regret of $f$ is defined as:
\[ \reg_{\alpha}(f(\X)) := \risk_\alpha(f(\X)) - \risk_{\alpha}^*.\]
\begin{lemma}
\label{lem:weightedriskequiv}
Let $\Psi$ be a linear-fractional performance metric as defined in \eqref{eqn:family1}, \eqref{eqn:family2} or \eqref{eqn:family3}.  Then for $\alpha \in (0,1)$ defined as:
\begin{equation}
\alpha = \frac{\Psi^*c_2 - c_1}{\Psi^*c_2 - c_1 + \Psi^*d_2 - d_1},
\label{eqn:alpha}
\end{equation} where $c_{1}, d_{1}, c_{2}, d_{2}$ are constants that depend on $\Psi$, there exists some constant $C > 0$ such that, for any $f$:
\begin{align}
\Psi^* - \Psi(f(\X), \Y) \leq C (\emph{Risk}_\alpha(f(\X)) - \emph{Risk}_{\alpha}^*).
\label{eqn:weightedriskequiv}
\end{align}
\end{lemma}

Let $\ell: \{0,1\}\times \mathbb{R} \to  \mathbb{R}_+$ be a $\lambda$-strongly proper composite loss~\citep{reid2010composite}, such as the squared loss or the logistic. Given real-valued predictions $\Z \in \mathbb{R}^{n\times L}$, we now argue that there exists a thresholding $\text{Thr}_{\theta^*}(\Z) \in \{0,1\}^{n\times L}$ such that $\risk_\alpha(\text{Thr}_{\theta^*}(\Z), \Y)$ is bounded by the $\ell$-regret of a strongly proper loss $\ell$ (where $\text{Thr}$ operator is defined as in Step 2 of Algorithm \ref{algo:algo1}). 
\begin{lemma}
Let $\ell$ be a $\lambda$-strongly proper loss function, and $\alpha$ be defined as in \eqref{eqn:alpha}. Then, there exists $\theta^*$ s.t. 
\[ \emph{Reg}_\alpha(\text{Thr}_{\theta^*}(\Z)) \leq \sqrt{\frac{2}{\lambda}} \sqrt{\emph{\reg}_\ell(\Z)} \ . \]
\label{lem:thresregretbound}
\end{lemma}
Finally, we show that estimating $\hat{\theta}$ from training samples (Step 3 of Algorithm \ref{algo:algo1}) is sufficient for bounding the $\Psi$-regret.
\begin{lemma} 
\label{lem:estimatetheta}
We have:
\[ \max_{\theta} \Psi(\emph{Thr}_\theta (\Z), \Y) \geq \Psi(\emph{Thr}_{\theta^*} (\Z), \Y), \]\\
and 
\[ \max_{\theta} \Psi(\emph{Thr}_\theta (\Z_{\Omega}), \Y_{\Omega}) \geq \max_{\theta} \Psi(\emph{Thr}_{\theta} (\Z), \Y) - O\bigg(\frac{1}{\sqrt{|\Omega|}}\bigg).\]
\end{lemma}
The proof of the Theorem is complete by chaining the above three Lemmas.  $\qed$
\begin{remark}
When $\Psi^*$ is known (in the noise-free or realizable setting, $\Psi^{*}$ is the maximum possible value of $\Psi$), we can get a closed form for $\theta^*$, which is $\theta^* = \xi(\alpha)$ where $\xi$ is the link function corresponding to the proper loss $\ell$. 
\end{remark}

\subsubsection{Proof of Lemma \ref{lem:weightedriskequiv}}
 Let $\Yh = f(\X)$. Consider the metric $\Psi$ from family \eqref{eqn:family1} for the moment. Define $A(\Yh) =  a_0 + a_{11} \tp + a_{01} \fp + a_{10} \fn + a_{00} \tn := c_1 \fp + d_1 \fn + e_1$ and $B(\Yh) = b_0 + b_{11} \tp + b_{01} \fp + b_{10} \fn + b_{00} \tn :=  c_2 \fp + d_2 \fn + e_2$ (for constants $c_1, c_2, d_1, d_2, e_1, e_2$ suitably defined), so that $\Psi(\Yh, \Y) = A(\Yh)/B(\Yh)$. Let $f^*$ denote the Bayes optimal attaining $\Psi^* = A^*/B^*$. We have:
\begin{align*}
\Psi^* - \Psi(\Yh, \Y) &= \frac{\Psi^*B(\Yh) - A(\Yh)}{B(\Yh)} \\
&= \frac{\Psi^*B(\Yh) - A(\Yh) - (\Psi^*B^* - A^*)}{B(\Yh)} \\
&= \frac{\Psi^*(B(\Yh) - B^*) - (A(\Yh) - A^*)}{B(\Yh)} \\
&= \frac{(\Psi^*c_2 - c_1) (\fp(\Yh, \Y) - \fp(f^*(\X), \Y)) + (\Psi^*d_2 - d_1) (\fn(\Yh, \Y)  - \fn(f^*(\X), \Y))}{B(\Yh)} \\
&\leq \frac{(\Psi^*c_2 - c_1) (\fp(\Yh, \Y) - \fp(f^*(\X), \Y))  + (\Psi^*d_2 - d_1)(\fn(\Yh, \Y) - \fn(f^*(\X), \Y)) }{\gamma} \\
&= C \big(\risk_\alpha(\Yh, \Y) - \risk_\alpha(f^*(\X), \Y)\big) \ .
\end{align*}
Assuming $(\Psi^*c_2 - c_1) \geq 0$ and $(\Psi^*d_2 - d_1) \geq 0$, the last equality follows by defining:
\begin{equation}
\alpha = \frac{\Psi^*c_2 - c_1}{\Psi^*c_2 - c_1 + \Psi^*d_2 - d_1}.
\end{equation} and $C = \frac{\Psi^*c_2 - c_1 + \Psi^*d_2 - d_1}{\gamma}$. The statement of the lemma follows. When $\Psi$ is a metric from family \eqref{eqn:family2}, we can apply Proposition 1 of \citep{koyejo2015consistent} to see that $\tp_i = \tp$, $\fp_i =\fp$ and so on (as the expectations are defined wrt $\tp_{ij}, \fp_{ij}$), which yields $\Psi^*$ is identical as in the micro-averaging case. So, the same regret bound applies as shown below: Define $A_i = a_0 + a_{11} \tp_i + a_{01} \fp_i + a_{10} \fn_i + a_{00} \tn_i = c_1 \fp_i + d_1 \fn_i + e_1$ and $B_i$ similarly. As before, let $\Psi^* = A^*/B^*$. So when $\Psi$ is of the form \eqref{eqn:family2},
\begin{align*}
\Psi^* - \Psi(\Yh, \Y) &= \frac{1}{n} \sum_{i=1}^n \frac{\Psi^*B_i(\Yh) - A_i(\Yh)}{B_i(\Yh)} \\
&= \frac{1}{n} \sum_{i=1}^n  \frac{\Psi^*B_i(\Yh) - A_i(\Yh) - (\Psi^*B^* - A^*)}{B_i(\Yh)} \\
&= \frac{1}{n} \sum_{i=1}^n \frac{\Psi^*(B_i(\Yh) - B^*) - (A_i(\Yh) - A^*)}{B_i(\Yh)} \\
&=\frac{1}{n} \sum_{i=1}^n \frac{(\Psi^*c_2 - c_1) (\fp_i(\Yh, \Y) - \fp(f^*(\X), \Y)) + (\Psi^*d_2 - d_1) (\fn_i(\Yh, \Y)  - \fn(f^*(\X), \Y))}{B_i(\Yh)} \\
&=\frac{1}{n} \sum_{i=1}^n \frac{(\Psi^*c_2 - c_1) (\fp(\Yh, \Y) - \fp(f^*(\X), \Y)) + (\Psi^*d_2 - d_1) (\fn(\Yh, \Y)  - \fn(f^*(\X), \Y))}{B_i(\Yh)} \\
&\leq \frac{(\Psi^*c_2 - c_1) (\fp(\Yh, \Y) - \fp(f^*(\X), \Y))  + (\Psi^*d_2 - d_1)(\fn(\Yh, \Y) - \fn(f^*(\X), \Y)) }{\gamma} \\
&= C \big( \risk_\alpha(\Yh, \Y) - \risk_\alpha(f^*(\X), \Y) \big)\ .
\end{align*}
which is identical to the bound for family \eqref{eqn:family1}. It is easy to see that \eqref{eqn:family3} also admits the above bound. Therefore, relation \eqref{eqn:weightedriskequiv} holds for all definitions of $\Psi$, with the same $\alpha$.

\subsubsection{Proof of Lemma \ref{lem:thresregretbound}}
Let $\Y, \Yh \in \{0,1\}^{n \times L}$. Note that for any $\ell$, $\risk_{\ell}(f)$ is defined as:
\[ \risk_{\ell}(f) = \E[\ell(\Yh_{ij}, \Y_{ij})] = \E_{\X \sim \bbP_\mathcal{\X}^{n}}\E_{(i,j)\sim \pi}\E_{\Y_{ij} \sim \bbP(.|\x_{i})}\ell(\Yh_{ij}, \Y_{ij}), \]
where $\pi$ denotes the sampling distribution over $(i,j)$ pairs.
Fix instance $i$ and label $j$. Let $\eta_{ij}$ denote the conditional probability of label $j$ of instance $i$ being 1, i.e. $\eta_{ij} = \bbP(\Y_{ij} = 1|\x_{i})$. For convenience, denote $\eta_{ij}$ simply by $\eta$.
Given $\eta \in [0,1]$, and $\hat{y} \in \{0,1\}$, consider the conditional $\ell_{\alpha}$-risk of $\hat{y}$:
\[ L_{\alpha}(\eta, \hat{y}) =  \alpha(1 - \eta)[[\hat{y} = 1]] + (1 - \alpha) \eta [[\hat{y}=0]],\]
and the corresponding conditional $\ell_{\alpha}$ regret of $\hat{y}$:
\[ \reg^{L}_{\alpha}(\eta, \hat{y}) =  L_{\alpha}(\eta, \hat{y}) - \min_{\hat{y}} L_{\alpha}(\eta, \hat{y}),\]
where we have: $\arg \min_{\hat{y}} L_{\alpha}(\eta, \hat{y}) = [[ \eta - \alpha ]].$\\
More generally, for a loss $\ell$, and a number $\hat{z}$, we have:
\[ L_{\ell}(\eta, \hat{z}) =  \ell(\hat{z}, 1) \eta + \ell(\hat{z}, 0) (1-\eta),\]
and
\[ \reg^{L}_{\ell}(\eta, \hat{z}) =  L_{\ell}(\eta, \hat{z}) - \min_{\hat{z}} L_{\ell}(\eta, \hat{z}).\]
Now, observe that:
\[ \risk_\alpha(\Yh, \Y) = \E_{\X \sim \bbP_\mathcal{\X}^{n}}\E_{(i,j)\sim \pi} L_{\alpha}(\eta_{ij}, \Yh_{ij}), \]
and 
\[ \reg_\alpha(\Yh, \Y) = \E_{\X \sim \bbP_\mathcal{\X}^{n}}\E_{(i,j)\sim \pi} \reg^{L}_{\alpha}(\eta_{ij}, \Yh_{ij}), \]
where the last equality follows from the fact that the Bayes optimal $f_{\alpha}^{*}$ of the $\ell_{\alpha}$-risk minimizes the conditional $L_{\alpha}(\eta_{ij}, .)$ risk for each $(i,j)$.
Let $\Z = f(\X) \in \mathbb{R}^{n\times L}$ denote real-valued predictions obtained using some function $f$. Using the same arguments as by \citet{kotlowski2015surrogate}, we can show that, by setting threshold $\theta^* = \xi(\alpha)$, where $\xi$ is the monotonic link function corresponding to $\lambda$-strongly proper loss $\ell$, and $\alpha$ is defined as in \eqref{eqn:alpha}, the conditional $\ell_{\alpha}$ regret of $\Yh_{ij} = [[ \Z_{ij} \geq \theta^* ]]$ for a fixed $(i,j)$ can be bounded as:
\[ \reg^{L}_{\alpha}(\eta_{ij}, \Yh_{ij}) \leq \sqrt{\frac{2}{\lambda}} \sqrt{\reg^{L}_{\ell}(\eta_{ij}, \Z_{ij})}, \]

Taking expectation wrt sampling distribution $\pi$ and the distribution over instances $\bbP_\mathcal{\X}^{n}$ on both the sides of the above inequality, and applying Jensen's inequality, the statement of the Lemma follows.

\subsubsection{Proof of Lemma \ref{lem:estimatetheta}}
The first part of the lemma is trivially true. For the second part, we can apply the same arguments as in Lemma 9 of \cite{nagaNIPS14}.

\subsection{Proof of Theorem \ref{thm:lafondbased}}
\label{app:lafondbased}
The following theorem bounds the error of the estimator $\hat{\W} \in \mathbb{R}^{n \times L}$ in this model, via the result by \citet{lafond2015}.
\begin{theorem} [~\citet{lafond2015}]
Assume $\pi$ is uniform, and consider the 1-bit matrix completion sampling model \eqref{eqn:samplingmodelmatcomp}. Let $\hat{\W}$ be the solution to the trace-norm regularized optimization problem \eqref{eqn:objective} using logistic loss for $\ell$ (with input $\X$ assumed to be identity matrix of size $n$), number of observations $|\Omega| \geq \log(n+L)\min(n,L)\max(c'_{\gamma}\log^2(c''_{\gamma}\sqrt{\min(n,L)}, 1/9)$,  and setting the regularization parameter
$\lambda =  2c_\gamma\sqrt{\frac{2\log(n+L)}{\min(n,L)|\Omega|}}$.
Then, with probability at least $1 - 3(n + L)^{-1}$, the following holds:
\[ \frac{\|\hat{\W} - \W^*\|^2_F}{nL} \leq \tilde{C} \max\bigg(\frac{\max(n,L) \rk(\W^*)\log (n+L)}{|\Omega|} \bigg(\sigma_\gamma^2 +1 \bigg), \frac{\gamma^2}{\mu}\sqrt{\frac{\log(n + L)}{|\Omega|}}\bigg), \]
where $\tilde{C}, c_\gamma, c'_\gamma, c''_\gamma, \sigma_\gamma$ are numerical constants.
\label{thm:lafond}
\end{theorem}
The above theorem can be extended to general distributions $\pi$ satisfying Assumption \ref{assume:omegadist}. See~\citet{lafond2015} for more details. Now, we use the fact that $\ell$ is 1-Lipschitz  (say, by choosing logistic loss), and bound $\E[\ell(\hat{\W}_{ij}, \Y_{ij}) - \ell(\W^*_{ij}, \Y_{ij})] \leq \frac{1}{nL}\sum_{ij} |\hat{\W}_{ij} - 
\W^*_{ij}|$. Observing that $\|\hat{\W} - \W^*\|_1 \leq \sqrt{nL}\|\hat{\W} - \W^*\|_F$, and combining with the bound in Theorem \ref{thm:lafond}, the proof is complete.

\subsection{Weakness of using \citet{lafond2015} for Multi-label Learning}
\label{sec:weaknesslafond}
In the multi-label learning model \eqref{eqn:samplingmodel}, one could hope to directly apply the analysis of~\citet{lafond2015} for recovering $\X\W^* \in \mathbb{R}^{n \times L}$, and in turn, $\W^* \in \mathbb{R}^{d \times L}$. In lieu of problem~\eqref{eqn:objective}, we would then solve the optimization problem in \citet{lafond2015}:
\begin{equation}
\hat{\W} = \arg \min_{\W: \|\X\W\|_\infty \leq \gamma} \frac{1}{|\Omega|} \sum_{(i,j) \in \Omega} \ell (\inner{\x_{i}}{\w_{j}}, \Y_{ij}) + \lambda \|\X\W\|_{*}
\label{eqn:lafondopt}
\end{equation}
Note that the only difference is how the trace-norm regularization is performed: $\|\X\W\|_*$ versus our proposed $\|\W\|_*$ in Algorithm \ref{algo:algo1}. The following corollary of Theorem \ref{thm:lafond} provides a bound for the recovery error of $\hat{\W}$.
\begin{corollary} 
Assume \ref{assume:features}, $\pi$ is uniform, and consider the sampling model \eqref{eqn:samplingmodel}. Let $\hat{\W}$ be the solution to the trace-norm regularized optimization problem \eqref{eqn:lafondopt} using logistic loss for $\ell$, number of observations $|\Omega| \geq \log(n+L)\min(n,L)\max(c'_{\gamma}\log^2(c''_{\gamma}\sqrt{\min(n,L)}, 1/9)$,  and setting the regularization parameter
$\lambda =  2c_\gamma\sqrt{\frac{2\log(n+L)}{\min(n,L)|\Omega|}}$.
Then, with probability at least $1 - 3(n + L)^{-1}$, the following holds:
\[ \frac{\|\hat{\W} - \W^*\|^2_F}{dL} \leq \frac{\tilde{C}}{d} \max\bigg(\frac{\max(n,L) \rk(\W^*)\log (n+L)}{|\Omega|} \bigg(\sigma_\gamma^2 +1 \bigg), \frac{\gamma^2}{\mu}\sqrt{\frac{\log(n + L)}{|\Omega|}}\bigg), \]
where $\tilde{C}, c_\gamma, c'_\gamma, c''_\gamma, \sigma_\gamma$ are numerical constants.
\label{thm:lafondcorollary}
\end{corollary}
When $n \geq L$ and $|\Omega| = O(n)$, which is quite common in multi-label scenario, the above bound suggests that $\hat{\W}$ from \eqref{eqn:lafondopt} is not even a consistent estimator, even when $\pi$ is uniform.  

\subsection{Proof of Theorem \ref{thm:logisticwithfeatures}}
The statement is a corollary of the more general Theorem \ref{thm:corollary}, proved in Appendix \ref{app:generalexponential}. We can compute the constants for the logistic loss as: $\bar{\sigma}_{\gamma} \leq 1$ and $\underline{\sigma}_{\gamma} \geq \frac{(1+e^{\gamma})^{2}}{e^{-\gamma}}$, over the domain $[-\gamma, \gamma]$.

\subsection{Proof of Theorem~\ref{thm:pucorollary}}
The following result by ~\citep{hsieh2015pu} gives recovery bound for the resulting estimator $\hat{\W}$, as described in the text (Section \ref{sec:pulearning}).
\begin{theorem}[\citep{hsieh2015pu}]
With probability at least $1 - 2(n+L)^{-1}$,
\[ \frac{\|\hat{\W} - \W^*\|^2_F}{nL} \leq  6\frac{\sqrt{\log(n+L)}}{\sqrt{nL}(1-\rho)} + 2 C \ . \ t \frac{\sqrt{n} + \sqrt{L}}{(1-\rho)nL}, \]
\end{theorem}
where $C$ is absolute constant and $\|\W^*\|_* \leq t$.
The proof is complete by using the same argument for $1$-Lipschitz $\ell$ as in the proof of Theorem \ref{thm:lafondbased}.
\newpage
\section{Sampling from Exponential Distribution}
\label{app:generalexponential}
We now consider the generalized matrix completion problem when the values are sampled iid from an exponential distribution parameterized by the input features $\x \in \mathbb{R}^d$. This setting extends that of \citet{lafond2015}. Let $y_{ij} \in \mathbb{R}$ denote a random sample corresponding to the user $i$ and label $j$, which is distributed as:
\begin{equation}
y_{ij} | \x_i, \w_j \sim \exp_{h, G}(\x_i, \w_j) := h(y_{ij}) \exp\big(\inner{\x_i}{\w_j}y_{ij} - G(\inner{\x_i}{\w_j})\big). 
\label{eqn:exponentialdist}
\end{equation}
where $\inner{\x_i}{\w_j}$, $i = 1,2,\dots,n$ and $j = 1,2,\dots,L$ are the canonical parameters, $h$ and $G$ are the base measure and log-partition functions associated with this canonical representation.

Let $\W^* \in \mathbb{R}^{d \times L}$ denote the ground-truth parameter matrix with $\w_j$'s as columns. Similarly, let $\Y \in \mathbb{R}^{n\times L}$  (with entries $y_{ij}$) denote a random sample from $\X\W^*$. As in the standard matrix completion setting, we only observe values of $\Y$ corresponding to a set of indices $\Omega$ sampled iid from a fixed distribution.

\paragraph{Notation.} With a slight abuse, we will continue to use $\inner{.}{.}$ when the arguments are matrices, instead of the \textbf{trace} operator, i.e. for matrices $A$ and $B$ of appropriate dimensions, $\inner{A}{B} := \textbf{trace}(A^TB)$. Let $\|A\|_\infty = \max_{ij}|A_{ij}|$, $\|A\|_F = \sqrt{\sum_{ij} A_{ij}^2}$,  $\|A\|_*$ denote the trace norm (sum of singular values of $A$), $\sigma_{\max}(A) = \|A\|_2$ denote the operator norm (maximum singular value of $A$), and $\sigma_{\min}(A)$ denote the smallest singular value.

\subsection*{Maximum Log-likelihood Estimator.} We consider the negative log-likelihood of the observations, given by:
\[ \Phi_Y(\X, \W) = -\frac{1}{|\Omega|}\sum_{(i,j) \in |\Omega|} y_{ij}\inner{\x_i}{\w_j}  - G(\inner{\x_i}{\w_j}). \]
Constrained ML estimator is obtained as:
\begin{equation} \hat{\W} := \arg \min_{\W: \|\X\W\|_{\infty} \leq \gamma} \Phi^{\lambda}_Y(\X, \W) := \Phi_Y(\X, \W) + \lambda \|\W\|_{*} 
\label{eqn:MLE}
\end{equation}

\begin{assumption}
\label{assume:G}
\begin{enumerate}
\item The function $G(x)$ is twice differentiable and strongly convex on $[-\gamma, \gamma]$, such that there exists constants $\bar{\sigma}_{\gamma} > 0$ and $\underline{\sigma}_{\gamma} > 0$ satisfying:
\[ \underline{\sigma}_{\gamma}^2 \leq G''(x) \leq \bar{\sigma}_{\gamma}^2, \]
for any $x \in [-\gamma, \gamma]$. 
\item There exists a constant $\delta_{\gamma} > 0$ such that for all $x \in [-\gamma, \gamma]$ and $y \sim \exp_{h,G}(x)$:
\[ \E_{y \sim \bbP(.|x)}\bigg[ \exp\bigg( \frac{|y - G'(x)|}{\delta_\gamma} \bigg)\bigg] \leq e. \]
\end{enumerate}
\end{assumption}

\begin{definition}
Given convex function $G(x)$ define the Bregman divergence between two scalars $x, x' \in \mathbb{R}$ as:
\begin{equation}
d_G(x, x') = G(x) - G(x') - G'(x')(x-x').
\label{eqn:bregdiv}
\end{equation}
\end{definition}

\begin{remark}
Under Assumption \ref{assume:G}.1, for any $x, x' \in [-\gamma, \gamma]$, the Bregman divergence $G$ satisfies:
\begin{equation} \underline{\sigma}_{\gamma}^2 (x-x')^2 \leq 2 d_G(x, x') \leq \bar{\sigma}_{\gamma}^2 (x-x')^2. 
\label{eqn:bregsqrel}
\end{equation}
\end{remark}

Let $E_{ij} \in \mathbb{R}^{n \times L}$ denote the indicator matrix with zeros everywhere except at $(i,j)$ where it is 1. For $(\epsilon_{ij})_{ij=1}^{|\Omega|}$ a Rademacher sequence independent from $(\Omega, Y_\Omega)$, define:
\begin{equation}
\Sigma_R := \frac{1}{|\Omega|} \sum_{(i,j)\in \Omega} \epsilon_{ij} E_{ij} .
\label{eqn:sigmar}
\end{equation}
\begin{theorem}
Assume \ref{assume:G}.1, \ref{assume:omegadist}.1, $\|\X\W^*\|_{\infty} \leq \gamma$, $\sigma_{\min}(\X) > 0$ and $2\|\X^T\nabla\Phi_\Y(\X, \W^*)\|_2 \leq \lambda$. Then, with probability at least $1 - 2(n + L)^{-1}$, the following holds:
\[ \frac{\|\hat{\W} - \W^*\|^2_F}{dL} \leq \frac{C\mu^2n}{\sigma^2_{min}(\X)\ .\ d} \max\bigg(L \rk(\W^*) \bigg(\frac{\lambda^2}{\underline{\sigma}_\gamma^4}\frac{n}{\sigma_{\min}^2(X)} + d \big(\ \E\|\Sigma_R\|_2\big)^2\bigg), \frac{\gamma^2}{\mu}\sqrt{\frac{\log(n + L)}{|\Omega|}}\bigg), \]
where $C$ is a numerical constant and $\Sigma_R$ is defined as in \eqref{eqn:sigmar}.
\label{thm:main1}
\end{theorem}
\begin{proof}
The proof closely follows that of Theorem 5 of \citet{lafond2015}. As $\hat{\W}$ is the minimizer of \eqref{eqn:MLE}, we have: 
\[ \Phi^\lambda_Y(X, \hat{\W}) - \Phi^\lambda_Y(X, \W^*) \leq 0 \]
It follows that:
\begin{eqnarray*}
 \lambda (\|\hat{W}\|_{*} - \|W^*\|_{*})  + \frac{1}{|\Omega|}\sum_{(i,j) \in \Omega}  y_{ij}\inner{\x_i}{\w^*_j - \hat{\w}_j} + G(\inner{\x_i}{\hat{\w}_j}) - G(\inner{\x_i}{\w^*_j})  \leq 0
\end{eqnarray*}
Using the fact that the gradient matrix:
\begin{equation}
  \nabla \Phi_Y(X, \W^*)  := \nabla_{X\W^*} \Phi_Y(X, \W^*) = -\frac{1}{|\Omega|}\sum_{(i,j) \in \Omega}\big(y_{ij} - G'(\inner{\x_i}{\w_j^*})E_{ij}
\label{eqn:gradient}
\end{equation}
(where $E_{ij}$ are the indicator matrices defined earlier) in the above inequality, we have:
\begin{align*}
 \lambda (\|\hat{\W}\|_{*} - \|\W^*\|_{*})  + \binner{\nabla \Phi_Y(X, \W^*)}{\X(\W^* - \hat{\W})} +\\ \frac{1}{|\Omega|}\sum_{(i,j)\in \Omega} G(\inner{\x_i}{\hat{\w}_j}) - G(\inner{\x_i}{\w_j^*}) - G'(\inner{\x_i}{\w_j^*})(\inner{\x_i}{\hat{\w}_j - \w_j^*})  \leq 0.
\end{align*}
Using the definition of the divergence \eqref{eqn:bregdiv}, and the fact that $\binner{\nabla \Phi_Y(X, \W^*)}{\X(\W^* - \hat{\W})} =  \binner{\X^T \nabla \Phi_Y(X, \W^*)}{\W^* - \hat{\W}}$ it follows that:
\[ D^{\Omega}_G(\X\hat{\W}, \X\W^*) := \frac{1}{|\Omega|} \sum_{(i,j) \in \Omega} d_G(\inner{\x_i}{\hat{\w}_j}, \inner{\x_i}{\w^*_j}) \leq  \lambda (\|\W^*\|_{*} - \hat{\W}\|_{*}) -  \binner{\X^T \nabla \Phi_Y(X, \W^*)}{\W^* - \hat{\W}}\]
The first term in the RHS of above inequality can be bounded first using Lemma 16-(iii) of \citet{lafond2015}. The second term can be bounded using the trace inequality (that uses the duality between $\|.\|_*$ and $\|.\|_2$) and the assumption on $\lambda$ stated in the Theorem. We get:
\[ D^{\Omega}_G(\X\hat{\W}, \X\W^*)  \leq  \lambda (\|\mP_{\W^*}(\hat{\W} - \W^*)\|_* + \frac{1}{2} \| \hat{\W} - \W^*\|_*). \]
To bound the first term in the above equation, we can apply Lemma 16-(ii) of \citet{lafond2015}. Lemma \ref{lem:lafondlemma17} gives a bound for the second term. Together we have:
\begin{equation} 
D^{\Omega}_G(\X\hat{\W}, \X\W^*)  \leq  3 \lambda \sqrt{2 \rk(\W^*)} \|\hat{\W} - \W^* \|_F. 
\label{eqn:upperbndD}
\end{equation}
By strong convexity of $G$ (Assumption \ref{assume:G}.1), we have:
\begin{equation}
\Delta_{\Y}^2(\X\hat{\W},\X\W^*) := \frac{1}{|\Omega|}\sum_{(i,j)\in \Omega}(\inner{\x_i}{\hat{\w}_j - \w^*_j})^2 \leq \frac{2}{\underline{\sigma}_{\gamma}^2} D^{\Omega}_G(\X\hat{\W},\X\W^*).
\label{eqn:DeltaDrel}
\end{equation}
Now, we will get a lower bound for $\Delta_Y^2(\X\hat{\W},\X\W^*)$. To do so, let us define $\beta := 8 e \gamma^2 \sqrt{\log(n + L)/|\Omega|}$ and distinguish the two following cases:
\paragraph{Case 1} If $\E[(\inner{\x_i}{\hat{\w}_j - \w^*_j})^2] \leq \beta$, where $\E$ is defined wrt the sampling distribution as in Assumption \ref{assume:omegadist}, then Lemma 18 of \citet{lafond2015} yields,
\begin{equation}
\frac{\|\X\hat{\W} - \X \W^*\|_F^2}{nL} \leq \mu \beta.
\label{eqn:case1}
\end{equation}
\paragraph{Case 2} If $\E[(\inner{\x_i}{\hat{\w}_j - \w^*_j})^2] > \beta$, consider $\hat{\W} \in \mathcal{C}(\beta, 32\mu d L\rk(\W^*))$, where $\mathcal{C}(.,.)$ is defined as:
\begin{equation}
\mathcal{C}(\beta, r) = \bigg\{ \W \in \mathbb{R}^{d \times L} \ |\  \|\W^* - \hat{\W}\|_* \leq \sqrt{r \E[\Delta_{\Y}^2(\X\W, \X\W^*)]};  \E[\Delta_{\Y}^2(\X\W, \X\W^*)] > \beta \bigg\}.
\label{eqn:cset}
\end{equation}
Then, from Lemma 19 of \citet{lafond2015}, it holds with probability at least $1 - 2(n+L)^{-1}$ that
\begin{equation}
\Delta_{\Y}^2(\X\hat{\W}, \X\W^*) \geq \frac{1}{2}  \E[\Delta_{\Y}^2(\X\hat{\W}, \X\W^*)] - 512 e (\E[\|\Sigma_R\|_2)^2\mu d L \rk(\W^*).
\end{equation}
Combining the above inequality with \eqref{eqn:DeltaDrel}, \eqref{eqn:upperbndD} and Lemma 18 of \citet{lafond2015} yields:
\[ \frac{\|\X\hat{\W} - \X\W^* \|_F^2}{2\mu nL} - 512e (\E[\|\Sigma_R\|_2)^2\mu d L \rk(\W^*) \leq \frac{6 \lambda}{\underline{\sigma}_\gamma^2}\sqrt{2 \rk(\W^*) } \|\hat{\W} - \W^*\|_F .\] 
We can use Lemma \eqref{lem:finallowerbound} to bound the first term from below. Applying the identity $ab \leq (a^2 + b^2)/4$, multiplying both sides of the inequality by $1/d$, rearranging and combining with \eqref{eqn:case1}, the proof is complete.
\end{proof}

\begin{theorem}
Assume \ref{assume:features}, \ref{assume:omegadist}, \ref{assume:G}. Choose, $n \geq C' \ . \ d$, $L \geq d$, $|\Omega| \geq L + d$ 
and
$\lambda =  \frac{2c\bar{\sigma}_{\gamma}}{\sqrt{|\Omega|}}$.
Then, with probability at least $1 - 3(n + L)^{-1} - 2(d+L)^{-1}$, the following holds:
\[ \frac{\|\hat{\W} - \W^*\|^2_F}{dL} \leq \frac{C_2\mu^2}{d} \max\bigg(\frac{L \rk(\W^*)\log (n+L)}{|\Omega|} \bigg(\frac{\bar{\sigma}_\gamma^2}{\underline{\sigma}_\gamma^4} +1 \bigg), \frac{\gamma^2}{\mu}\sqrt{\frac{\log(n + L)}{|\Omega|}}\bigg), \]
where $c,C', C_2$ are numerical constants.
\label{thm:corollary}
\end{theorem}
\begin{proof}
It suffices to show $2\|\X^T\nabla \Phi(\X,\W^*)\|_2 \leq \lambda$ for chosen $\lambda$ in the statement of the Theorem and a suitable bound for $\E \|\Sigma_R\|_2$ (the result would then follow by applying Theorem \ref{thm:main1}). The latter term can be readily bounded applying the corresponding arguments in the proof of Theorem 6 of \citet{lafond2015}, which yields:
\begin{equation}
\E\|\Sigma_R\|_2 \leq c^*\sqrt{\frac{2e\log(n+L)}{|\Omega|}\bigg(\frac{\nu}{\min(n,L)}\bigg)},
\label{eqn:sigmarbound}
\end{equation}
where we use the fact that $\sum_{l=1}^L \pi_{k,l} = \frac{\nu}{\min(n,L)}$ (by Assumption \ref{assume:omegadist}).
where $c^*$ is a numerical constant. \\
We can apply Lemma \ref{lem:gradbound} to bound  $\|\X^T\nabla \Phi(\X,\W^*)\|_2$, with the $\lambda$ chosen in the statement of the Theorem. The proof is complete noting that for the choice of $n$ as in the statement of the Theorem, Lemma \ref{lem:sigmaxbound} implies $\sigma_{\min}^2(X) \geq \underline{C} n$ and that for the choice of $n$ and $L$ as in the statement of the Theorem, $\frac{d}{\min(n,L)} \leq 1$.

\end{proof}

\begin{lemma}
Let $\X\W, \X\tilde{\W} \in \mathbb{R}^{n \times L}$ satisfy $\|\X\W\|_\infty \leq \gamma$ and $\|X\tilde{\W}\|_\infty \leq \gamma$. Assume: \\ $2\|\X^T\nabla\Phi_\Y(\X, \tilde{\W})\|_2 \leq \lambda$, and $\Phi^\lambda_Y(X, \W) \leq \Phi^\lambda_Y(X, \tilde{\W})$. Then:\\
(i) $\|\mP_{\tilde{\W}}^{\perp}(\W - \tilde{\W})\|_* \leq 3 \|\mP_{\tilde{\W}}(\W - \tilde{\W})\|_*$, \\
(ii) $\|\W - \tilde{\W}\|_* \leq 4\sqrt{2\rk(\tilde{\W})}\|\W - \tilde{\W}\|_F$.
\label{lem:lafondlemma17}
\end{lemma}
\begin{proof}
The proof closely follows that of Lemma 17 of~\citep{lafond2015}. By definition, we have:
\[ \Phi^\lambda_Y(\X, \W) - \Phi^\lambda_Y(\X, \tilde{\W}) \leq 0 \]
or,
\[  \Phi_Y(\X, \W) - \Phi_Y(\X, \tilde{\W}) \leq \lambda(\| \tilde{\W} - \W\|_{*} ) \ .\]
Writing $\W \in \mathbb{R}^{d \times L}$ as $\W = \tilde{\W} + \mathcal{P}_{\tilde{\W}}^{\perp}(\W - \tilde{\W}) + \mathcal{P}_{\tilde{\W}}(\W - \tilde{\W})$, Lemma 16-(i) of~\citep{lafond2015} and triangle inequality together give:
\[ \|\W\|_{*} \geq  \|\tilde{\W}\|_{*} + \|\mathcal{P}_{\tilde{\W}}^{\perp}(\W - \tilde{\W}) \|_{*}+ \|\mathcal{P}_{\tilde{\W}}(\W - \tilde{\W})\|_{*}, \]
Or,
\begin{equation}
\Phi_Y(\X, \tilde{\W}) - \Phi_Y(\X, \W) \geq \lambda(\|\mathcal{P}_{\tilde{\W}}^{\perp}(\W - \tilde{\W}) \|_{*}+ \|\mathcal{P}_{\tilde{\W}}(\W - \tilde{\W})\|_{*} ) \ . 
\label{eqn:interstep1}
\end{equation}
Note that by convexity of $\Phi_{Y}$:
\[ \Phi_Y(\X, \tilde{\W}) - \Phi_Y(\X, \W) \leq \binner{\nabla \Phi_{Y}(\X, \tilde{\W})}{\X\tilde{\W} - \X\W} = \binner{\X^{T}\nabla \Phi_{Y}(\X, \tilde{\W})}{\tilde{\W} - \W},  \]
By trace inequality, we have: 
\[ \Phi_Y(\X, \tilde{\W}) - \Phi_Y(\X, \W) \leq \|\X^{T}\nabla \Phi_{Y}(\X, \tilde{\W})\|_{2}\|\tilde{\W} - \W\|_{*} \leq \frac{\lambda}{2}\|\tilde{\W} - \W\|_{*} \]
where the last inequality is by assumption, $\|\X^{T}\nabla \Phi_{Y}(\X, \tilde{\W})\|_{2} \leq \lambda/2$. The last term in the above inequality can be bounded by $\frac{\lambda}{2}\bigg(\|\mathcal{P}_{\tilde{\W}}^{\perp}(\W - \tilde{\W})\|_{*} +\| \mathcal{P}_{\tilde{\W}}(\W - \tilde{\W}) \|_{*}\bigg)$. Together with \eqref{eqn:interstep1}, we get the first part of the Lemma. We can now conclude the proof of part two using identical arguments as in Lemma 17 of~\citep{lafond2015}.
\end{proof}

\begin{lemma}
Let $\sigma_{\min}(\X)$ denote the smallest singular value of $\X$. Then for any $\W, \tilde{\W}$, Then:
\[ \|\X\W - \X\tilde{\W} \|_F^2 \geq \sigma^2_{\min}(\X) \|\W - \tilde{\W} \|_F^2. \]
\label{lem:finallowerbound}
\end{lemma}
\begin{proof}
Observe that $\|\X(\W - \tilde{\W}) \|_F^2 = \textbf{trace}\big(\X(\W - \tilde{\W})(\W - \tilde{\W})^T\X^T\big) = \textbf{trace}\big((\W - \tilde{\W})(\W - \tilde{\W})^T\X^T\X\big) \geq \sigma_{\min}(\X^T\X)\textbf{trace}\big((\W - \tilde{\W})(\W - \tilde{\W})^T\big) = \sigma_{\min}(\X)^2 \|\W-\tilde{\W}\|_F^2$.
\end{proof}

\begin{lemma}
Let $\X \in \mathbb{R}^{n \times d}$ be a matrix with rows sampled from sub-Gaussian distribution satisfying Assumption \ref{assume:features}. Furthermore, choose:
\[ n \geq C' d \ .\]
 Then, with probability at least $1 - 2e^{-d}$, each of the following statements is true:
\[ \sigma_{\max}(\X^T \X) \leq \bar{C} n, \]
\[ \sigma_{\min}(\X^T \X) \geq \underline{C} n, \]
where $C', \bar{C}$ and $\underline{C}$ are absolute constants that depend only on the parameters $K$ and $\Sigma$ of the sub-Gaussian distribution.
\label{lem:sigmaxbound}
\end{lemma}
\begin{proof}
Using Lemma 16 of \citet{bhatia2015}, we have for any $\delta > 0$, with probability at least $1 - \delta$, each of the following statements hold:
\[ \sigma_{\max}(\X^T\X) \leq \sigma_{\max}(\Sigma) \ . \ n + C_K \sqrt{dn} + t \sqrt{n}, \]
\[ \sigma_{\min}(\X^T\X) \geq \sigma_{\min}(\Sigma) \ . \ n - C_K \sqrt{dn} - t \sqrt{n}, \]
where $t = \sqrt{\frac{1}{c_K}\log \frac{2}{\delta}}$, and $c_K$, $C_K$ are absolute constants that depend only on the sub-Gaussian norm $K$ of the distribution $\bbP_{\mathcal{X}}$. Now, choosing $\delta = 2e^{-d}$ or $\log(2/\delta) = d$, we have:
\[ C_K \sqrt{dn} + t \sqrt{n} = C_K \sqrt{dn} +  \sqrt{\frac{1}{c_K}dn} = \sqrt{dn}\bigg(C_K + \sqrt{\frac{1}{c_k}}\bigg).\]
For ease, define $C'_K := C_K + \sqrt{\frac{1}{c_k}}$. Now, choosing $n \geq \big(\frac{C'_K}{\sigma_{\min}(\Sigma)}\big)^2 \ .\ d$, and substituting above we have:
\[ C_K \sqrt{dn} + t \sqrt{n} \leq \frac{1}{2}\sigma_{\min}(\Sigma) \ . n. \]
Therefore:
\[ \sigma_{\max}(\X^T \X) \leq \bigg(\sigma_{\max}(\Sigma) + \frac{1}{2} \sigma_{\min}(\Sigma)\bigg)n, \]
\[ \sigma_{\min}(\X^T \X) \geq  \frac{1}{2} \sigma_{\min}(\Sigma) \ . \ n. \]
The proof is complete.
\end{proof}

\subsection*{Proof of Lemma \ref{lem:gradbound}}
Let $H$ denote the matrix with $h_{ij} = y_{ij} - G'(\inner{\x_i}{\w^*_j})$. Let $\h^i$ denote the $i$th row of $H$. Let $\mathcal{P}_{\Omega}(H)$ denote the projection of $H$ onto the observed indices $\Omega$. Let $\Omega_i$ denote the observed indices in row $i$ of $\Y$. For a vector $\bv$, let $\bv_{\Omega_i}$ denote its projection onto the observed indices $\Omega_i$. 

Fix $\bu \in \mathbb{R}^d$ and $\bv \in \mathbb{R}^L$. Define $a_i = \x_i^T \bu$ and $b_i = \inner{\bv_{\Omega_i}}{\h_{\Omega_i}^i}$. We have:
\begin{align*}
 \frac{1}{|\Omega|} \bu^T \X^T \mathcal{P}_{\Omega}(H) \bv &= \frac{1}{|\Omega|} \sum_{i=1}^n a_i b_i \\
&= \frac{1}{|\Omega|} \sum_{i=1}^n \|\bv_{\Omega_i}\|_2\ .\ a_i \frac{b_i}{\|\bv_{\Omega_i}\|_2}.
\end{align*}
Consider $b_i = \sum_{(i,j) \in \Omega} v_j h_{ij}$. Note that $h_{ij}$'s are sub-Gaussian random variables with sub-Gaussian norm $\alpha$. Using Lemma 5.9 of \citet{vershynin2010introduction}, we have $b_i$ is sub-Gaussian with norm $\|\bv_{\Omega_i}\|_2\alpha$. In turn, this implies, $\frac{b_i}{\|\bv_{\Omega_i}\|_2}$ is sub-Gaussian with sub-Gaussian norm $\alpha$. Therefore, $\frac{a_ib_i}{\|\bv_{\Omega_i}\|_2}$ is $\alpha$-subexponential. Applying Proposition 5.16 of \citet{vershynin2010introduction}, we have, with probability at least $1 - \delta$,
\[ \frac{1}{|\Omega|} \sum_{i=1}^n  \|\bv_{\Omega_i}\|_2\ .\ a_i \frac{b_i}{\|\bv_{\Omega_i}\|_2} \leq \frac{c \ . \ \alpha}{|\Omega|}\bigg(\sqrt{\sum_{i=1}^n \|\bv_{\Omega_i}\|^2}\sqrt{\log \frac{2}{\delta}} + \max_{i \in [n]} \|\bv_{\Omega_i}\|^2 \log \frac{2}{\delta}  \bigg) \ .\]
for some absolute constant $c$. Noting that: $\|\bv\|_2 = 1$ and for any $j \in [L]$, $|\{i : (i,j) \in \Omega\}| \leq \frac{c' . |\Omega|}{L}$, we have, with probability at least $1 - \delta$,
\[ \frac{1}{|\Omega|} \sum_{i=1}^n  \|\bv_{\Omega_i}\|_2\ .\ a_i \frac{b_i}{\|\bv_{\Omega_i}\|_2} \leq \frac{c \ . \ \alpha}{|\Omega|}\bigg(\sqrt{\frac{c' . |\Omega|}{L}}\sqrt{\log \frac{2}{\delta}} + \log \frac{2}{\delta}  \bigg) \ .\]
We conclude the proof by a covering argument: Taking a union bound over $\epsilon$-ball of $\bu$ and $\bv$, we have, with probability at least $1 - (d+L)^{-1}$:
\[\big\| \X^T \nabla \Phi_Y(\X, \W^*) \big\|_2  \leq \frac{c \ . \ \alpha}{|\Omega|}\bigg(\sqrt{\frac{c' . |\Omega|}{L}}\sqrt{d + L} + d + L \bigg) \ .\]
Assuming $d \leq L$ and $|\Omega| \geq (L+d)$, the proof is complete.

\end{document}